\documentclass[runningheads]{packages/llncs}
\usepackage{graphicx}

\usepackage{cite}
\usepackage{amsmath,amssymb,amsfonts}
\usepackage{graphicx}
\usepackage{textcomp}
\usepackage{xcolor}
\usepackage{subcaption}
\usepackage{caption}
\definecolor{kellygreen}{rgb}{0.3, 0.73, 0.09}
\usepackage{pifont}
\newcommand{\cmark}{\textcolor{kellygreen}{\ding{51}}}%
\newcommand{\xmark}{\textcolor{red}{\ding{55}}}%
\usepackage{tikz,lmodern}
\usepackage{bbm}
\usepackage{times}
\usepackage{url}
\usepackage[hidelinks]{hyperref}
\usepackage{soul}

\usepackage{xcolor, soul}
\usepackage{booktabs,tabularx,ragged2e}
\usepackage{rotating}
\usepackage{multirow}
\usepackage{color, colortbl}
\definecolor{colorA}{rgb}{0.5, 0.0, 1.0}
\definecolor{colorB}{rgb}{0.2490, 0.3841, 0.9806}
\definecolor{colorC}{rgb}{0.0019, 0.7092, 0.9232}
\definecolor{colorD}{rgb}{0.2529, 0.9256, 0.8301}
\definecolor{colorE}{rgb}{0.5039, 0.9999, 0.7049}
\definecolor{colorF}{rgb}{0.7549, 0.9209, 0.5523}
\definecolor{colorG}{rgb}{1.0, 0.7005, 0.37841}
\definecolor{colorH}{rgb}{1.0, 0.0, 0.0}
\definecolor{colorlight}{rgb}{0.67, 0.88, 0.69}
\newcommand{\comment}[1]{}

\usepackage{algorithm}
\usepackage[noend]{algpseudocode}
\newcolumntype{C}[1]{>{\centering}p{#1}}

\begin{document}
\title{Unsupervised Features Ranking via Coalitional Game Theory for Categorical Data\thanks{This research was supported by the research training group \emph{Dataninja} (Trustworthy AI for Seamless Problem Solving: Next Generation Intelligence Joins Robust Data Analysis) funded by the German federal state of North Rhine-Westphalia.}}
%
%
\author{
Chiara Balestra \inst{1}
\and Florian Huber\inst{3}
\and Andreas Mayr\inst{2}
\and Emmanuel Müller\inst{1}}
\authorrunning{Balestra et al.
}
%
\institute{TU Dortmund, Germany \and
Department of Medical Biometry, Informatics 
and Epidemiology,\\ University Hospital of Bonn, Germany\and
University of Bonn, Germany}
\maketitle 

\begin{abstract}
Not all real-world data are labeled, and when labels are not available, it is often costly to obtain them. Moreover, as many algorithms suffer from the curse of dimensionality, reducing the features in the data to a smaller set is often of great utility. Unsupervised feature selection aims to reduce the number of features, often using feature importance scores to quantify the relevancy of single features to the task at hand. These scores can be based only on the distribution of variables and the quantification of their interactions. The previous literature, mainly investigating anomaly detection and clusters, fails to address the redundancy-elimination issue. We propose an evaluation of correlations among features to compute feature importance scores representing the contribution of single features in explaining the dataset's structure.

Based on Coalitional Game Theory, our feature importance scores include a notion of redundancy awareness making them a tool to achieve redundancy-free feature selection. We show that the deriving features' selection outperforms competing methods in lowering the redundancy rate while maximizing the information contained in the data. We also introduce an approximated version of the algorithm to reduce the complexity of Shapley values' computations.

\end{abstract}


\section{Introduction} \label{intro}
In machine learning, both feature selection methods and reduction of dimensionality are often performed to increase interpretability and to reduce computational complexity. As an example, for unsupervised applications such as clustering~\cite{cheng} or anomaly detection~\cite{cmi}, the curse of dimensionality poses a major challenge. Unsupervised feature selection enables the detection of data patterns, as well as the description of these patterns using a concise set of relevant features~\cite{unsup_fs,ufs}. The corresponding methods are mostly based on the analysis of multivariate data distributions, pairwise correlations, higher-order interactions among features, or pseudo-labels. The use of such complex measures implies that both the selection as well as the interpretation of why some features have been selected is challenging. On the one hand, selection requires basic measures to quantify the interaction within a set of features~\cite{cheng,cmi,mic}. On the other hand, interpretation of higher-order interactions is non-straight-forward and requires the decomposition of complex non-linear, higher-order, and multivariate measures to feature importance scores. \par
\begin{figure}[!t]
\subcaptionbox{\label{figurA}}[0.65\linewidth]{\includegraphics[height=3cm]{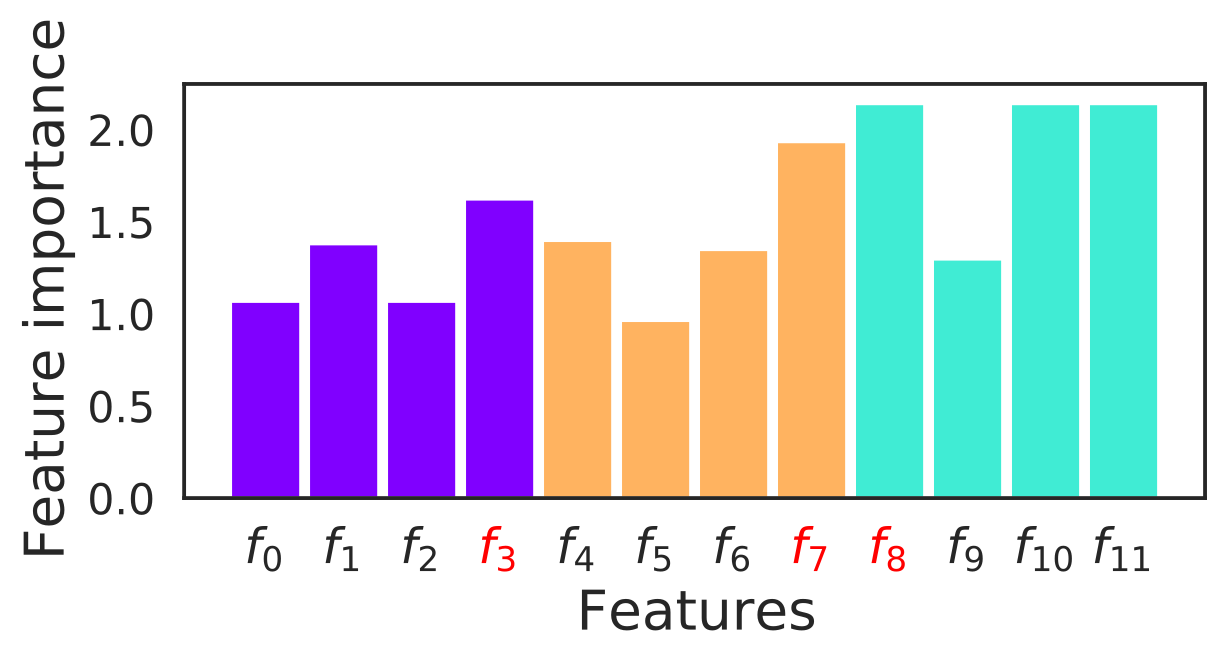}}
\subcaptionbox{\label{figurB}}[0.33\linewidth]{\includegraphics[height=3cm]{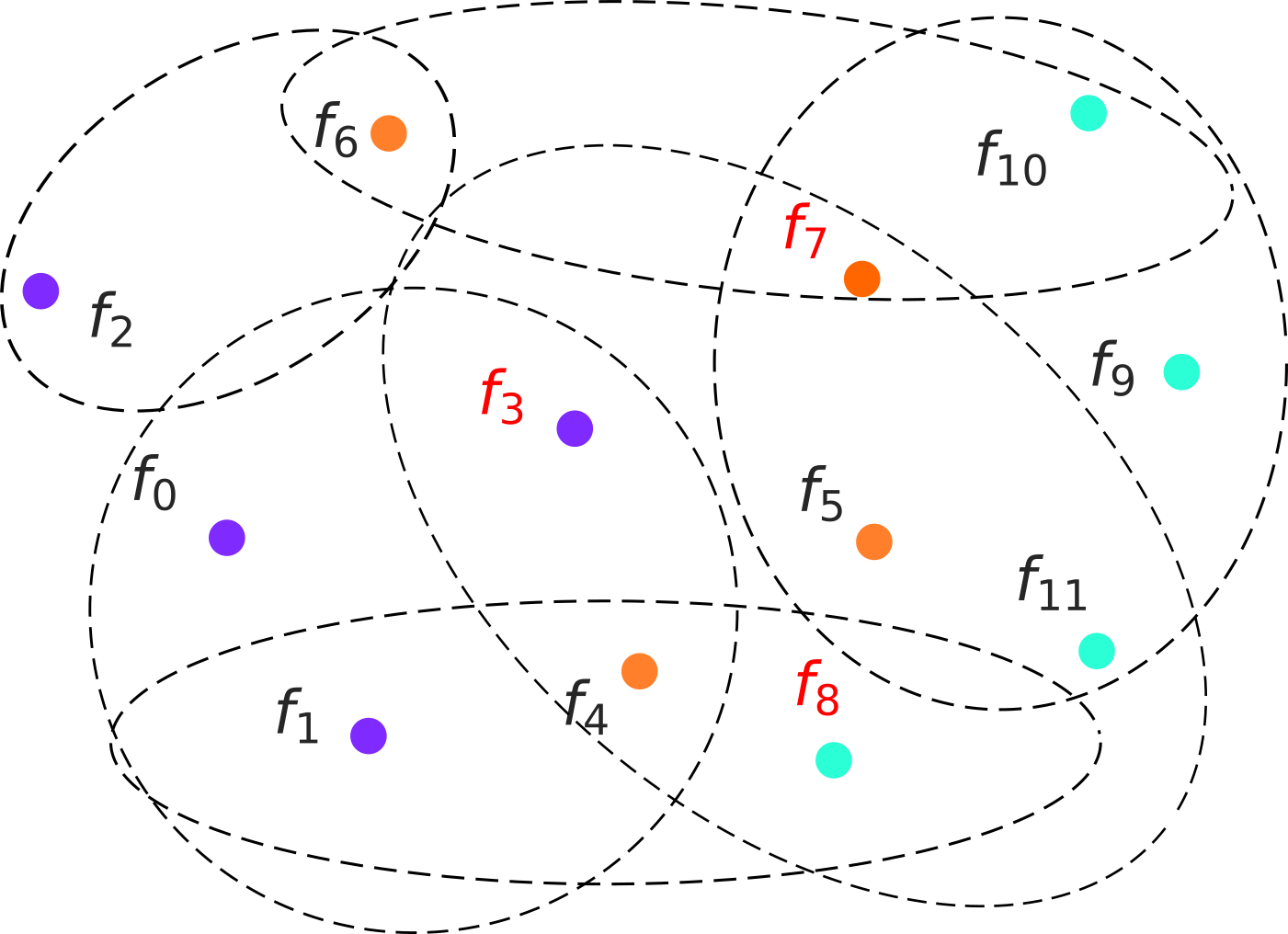}}
\caption{\label{fig:ranking_synth}(a) Unsupervised Shapley values-based feature importance scores. (b) Shapley values consider interactions within all possible subsets of features $f_i$s. In both figures, correlated subsets of features are color-coded and features selected by the proposed algorithm are marked in red.}
\vspace{-0.65cm}
\end{figure}
In different application domains, raising understanding over the mechanisms of underlying machine learning techniques has become a crescent necessity. Assigning scores to features based on their contributions to the machine learning procedure plays a decisive role to this end. Feature importance scores are prevalent in supervised learning, e.g., random forests. At the same time, for unsupervised tasks, the literature is limited either to traditional scores~\cite{unsup_fs,ufs} not sensitive to higher-order interactions, or the scores are not easily interpretable higher-order correlation measures.
\par
We propose new unsupervised feature importance scores decomposing the information contained in the data using axiomatic game-theoretic properties. In particular, Shapley values enable us to consider the interactions present in each possible subset of features (Figure~\ref{fig:ranking_synth}(b)) and to assign importance scores to the single features accordingly. Our approach consists of two steps. In the first step, we introduce a game-theoretic solution to decompose the information contained in the dataset and assign importance scores to the single features. The scores obtained consider complex higher-order feature interactions, can be based on different correlation measures, and do not rely on specific notions of clustering or anomalies. In particular, we use Shapley values~\cite{shapley} to get the feature importance scores where features explaining the most information on the overall dataset obtain a higher score. In the second step, we take care of a mechanism to reduce the redundancy among features. To this end, feature importance scores are penalized through an information-theoretic measure of correlation to yield a redundancy-free feature selection. Figure~\ref{fig:ranking_synth}(a) displays how correlated features are ranked similarly before applying the redundancy elimination step and how our method is capable of avoiding selecting highly correlated features. In the experimental results, we show that our ranking is achieving a redundancy free-ranking; the redundancy rate of the selected features is kept low both in synthetic and real datasets. \par
As a final remark, the scores flexibly rely on different correlation measures and are not bound to any clustering or anomaly detection goals. We choose to present the results obtained using the total correlation; hence, the presented experimental results are limited to discrete and categorical data. The procedure can be extended to mixed datasets replacing the total correlation or applying discretization on continuous data.
\begin{table*}[!t]
	\centering
	\begin{tabularx}{\textwidth} { >{\centering\arraybackslash}X |>{\centering\arraybackslash}X >{\centering\arraybackslash}X >{\centering\arraybackslash}X >{\centering\arraybackslash}X >{\centering\arraybackslash}X >{\centering\arraybackslash}X >{\centering\arraybackslash}X }
		& versatile quality notion & feature ordering & iterative selection & redundancy awareness & higher-order interactions
		\\\hline
		UDFS~\cite{UDFS}& \xmark & \cmark & \xmark & \cmark &\cmark\\
		MCFS~\cite{MCFS} & \xmark & \cmark & \xmark & \cmark & \xmark\\
		NDFS~\cite{NDFS} & \xmark & \cmark & \xmark & \cmark & \xmark\\
		SPEC~\cite{SPEC} & \cmark & \cmark & \xmark & \xmark & \xmark\\
		LS~\cite{LS} & \cmark & \cmark & \xmark & \xmark & \xmark \\
		PFA~\cite{PFA} & \cmark & \xmark & \xmark & \cmark & \xmark\\
		FSFC~\cite{FSFC} & \xmark & \xmark & \cmark & \cmark &\xmark\\
		\textbf{\footnotesize{this paper}}& \cmark & \cmark & \cmark & \cmark & \cmark \\
	\end{tabularx}
	\vspace*{3mm}
	\caption{Summary table of the competing methods and this paper.}
	\label{tab:comparing}
	\vspace{-10mm}

\end{table*}

\section{Related work}
Dimensionality reduction helps avoid the curse of dimensionality and increases the interpretability of data and machine learning techniques. Different methods analyze the relationship among features, the class label, and the correlation among variables~\cite{MI} and get feature importance scores in order to allow for a more aware use of machine learning by non-experts. Those scores are often not aware of correlations among variables, thus leading to a necessary integration of a redundancy awareness concept~\cite{rar}. \par
In 2007 game theory found application in supervised feature selection~\cite{hullermeier, cohen} where the value function was defined as the accuracy or the generalization error of the trained model; to the best of our knowledge, the approaches proposed in the recent years are limited to labeled data. A recent paper~\cite{rozemberczki2022shapley} underlined how Shapley values spread through machine learning; in particular, they appear in several techniques to increase the overall interpretability of black-box models~\cite{shap,strumbelj} and new insights on Shapley values and their applications continue appearing in the literature~\cite{catav21a}. \par 
As a downside, Shapley values are well known to be computationally expensive. Several approximations found place in the literature, e.g., \cite{approx1,approx2,burgess_approximating_2021} among others; the first attempt of a comprehensive survey of Shapley values' approximations is represented by Rozemberczki et al.~\cite{rozemberczki2022shapley}. To reduce the computational run-time, we implement Castro et al.~\cite{approx2} approximation, i.e., the most common Shapley values' approximation non relying on additional assumptions on the players. \par 
As a parallel area of research, in recent years, unsupervised feature selection methods have raised strong interest in the community~\cite{ufs,li:algorithms}. We selected a representative sample within the vast number of unsupervised feature selection methods to compare the performance of our approach. Among them, UDFS~\cite{UDFS} creates pseudo-labels to perform the feature selection in unlabelled data; MCFS and NDFS~\cite{MCFS, NDFS} concentrate on keeping the clustering structure. LS~\cite{LS} selects features by their local preserving power. PFA~\cite{PFA} tries to eliminate the downside of PCA while keeping the information within the data. Most of these algorithms tend to select features as a by-product of retaining a clustering structure in the data. Finally, FSFC~\cite{FSFC} is meant to select only non-redundant variables using a new definition of distance in the $k$-nearest neighbors. Table~\ref{tab:comparing} illustrates a summary of the properties of the various methods analyzed in comparison with our paper. 

\section{Feature importance measures}
Consider a $N$-dimensional dataset containing $D$ instances. We interpret each dimension as the realization set of a random variable, refer to the set of variables as $\mathcal{F}=\{X_1,\ldots,X_N\}$ and to each dimension $X_i$ as $i$th feature or variable. Feature selection methods often internally assign to subsets of features an importance score and output the subset maximizing the mentioned score. We propose to rank features considering their average contribution to all the possible subsets of features. The higher the average contribution of a feature is, the more convenient it is to keep it within the selected features. Additionally, we will also introduce redundancy awareness in these scores. \par
Given a function that assigns a value to each subset of features, assessing the \emph{importance} of single features is not trivial as each feature belongs to $2^{N-1}$ subsets of features. In unsupervised contexts, we can assess the usefulness of a set of features measuring correlations or clustering properties. Throughout the manuscript, we stick to a value function that captures the maximal \emph{information} contained in the data. Following this choice, the approach presented is restricted to categorical tabular data. We compute feature importance scores and obtain a ranking prioritizing features highly correlated with the rest of the dataset.

\subsection{Feature importance score}
We obtain feature importance scores using coalitional game theory. Each game is fully represented by the set of players $\mathcal{F}$ and a set function $v$ that maps each subset $\mathcal{A}\subseteq\mathcal{F}$ to $v(\mathcal{A})\in \mathbbm{R}$. $v$ is referred to as \emph{value function}~\cite{shapley} and satisfies the following properties
\begin{enumerate}
 \item $v(\emptyset) = 0$,
 \item $v(\mathcal{A}) \geq 0$ for any $\mathcal{A}\subseteq\mathcal{F}$, and
 \item $v(\mathcal{A})\leq v(\mathcal{B})$ for any $\mathcal{A},\mathcal{B}\subseteq\mathcal{F}$ such that $\mathcal{A}\subseteq \mathcal{B}$.
\end{enumerate} 
Working with unlabelled data, we can not rely on ground truth labels. Hence, we define value functions relying on intrinsic properties of the dataset; we opt for a value function measuring the independence of the features in $\mathcal{A}\subseteq\mathcal{F}$. One possible initialization for $v$ is the \emph{total correlation} of $\mathcal{A}$.
\begin{definition}
	The \emph{total correlation} $C$ of a set of variables $\mathcal{A}\subseteq \mathcal{F}$ is defined as
	\begin{equation}
	C(\mathcal{A})=\sum_{X\in \mathcal{A}}H(X)-H(\mathcal{A}).
	\end{equation} $H(\mathcal{A})$ is the Shannon entropy of the subset of discrete random variables $\mathcal{A}$, i.e., 
	\begin{equation}
	 H(\mathcal{A})=-\sum_{\Vec{x}\in \mathcal{A}}p_\mathcal{A}(\Vec{x})\log p_\mathcal{A}(\Vec{x})
	\end{equation}
	where $p_\mathcal{A}(\cdot)$ is the joint probability mass function of $\mathcal{A}$. \\ $H(X)$ is the Shannon entropy of $X$, i.e., $H(X)=-\sum_{{x}\in X}p_X({x})\log p_X({x}).$
\end{definition}
We choose the total correlation as it satisfies properties (2) and (3), it has an intuitive meaning and can be easily extended such that it satisfies property (1).
\par
Shannon entropy~\cite{shannon} measures the uncertainty contained in a random variable $X$ considering how uniform data are distributed: its value is close to zero when its probability mass function $p_X$ is highly skewed while, as the distribution approaches a uniform distribution, its value increases. Moreover, the Shannon entropy is a monotone non-negative function and can be extended such that $H(\emptyset)=0$. We assume that all features in $\mathcal{F}$ are discrete as the extension of Shannon entropy to continuous variables is not monotone~\cite{diff_entropy}. As a consequence of Shannon entropy's properties, the total correlation $C(\mathcal{A})$ is close to zero if the variables in $\mathcal{A}$ are independent, and it increases when they are correlated. To study the impact of adding a feature $Y$ to $\mathcal{A}\subseteq\mathcal{F}$, we compute the value function of the incremented subset $v(\mathcal{A}\cup Y)$ and compare it with $v(\mathcal{A})$: The difference $v(\mathcal{A}\cup Y)-v(\mathcal{A}) = H(\mathcal{A})+H(Y)-H(\mathcal{A}\cup Y)$ is non-negative and measures how much $\mathcal{A}$ and $Y$ are correlated. We refer to $H(\mathcal{A})+H(Y)-H(\mathcal{A}\cup Y)$ as \emph{marginal contribution of $Y$ to $\mathcal{A}$}. If $\mathcal{A}$ and $Y$ are independent, then the marginal contribution of $Y$ to $\mathcal{A}$ equals zero. Vice versa, the marginal contribution grows the stronger the correlation between $Y$ and $\mathcal{A}$ is. As importance score, we assign to $X_i$ the average of its marginal contributions and we refer to it as $\phi(X_i)$, i.e., 
\begin{equation}\label{eqn:shapley}
 \phi(X_i)=\sum_{\mathcal{A}\subseteq\mathcal{F}\setminus X_i}\frac{1}{N\binom{N-1}{|\mathcal{A}|}}[H(\mathcal{A})+H(X_i)- H(\mathcal{A}\cup X_i)]
\end{equation} 
corresponding to the \emph{Shapley value} of the player $X_i$ in the game $(\mathcal{F}, v)$ when $v$ is the total correlation. The general definition of Shapley values reads~\cite{shapley}:
\begin{definition}
	Given a coalitional game $(\mathcal{F},v)$ and a player $X_i\in\mathcal{F}$, the \emph{Shapley value of $X_i$} is defined by
	\begin{equation*}
	\phi_v(X_i)= \sum_{\mathcal{A}\subseteq\mathcal{F}\setminus {X_i}}\frac{1}{N\binom{N-1}{|\mathcal{A}|}}\left[v(\mathcal{A}\cup X_i)-v(\mathcal{A})\right].
	\end{equation*}
\end{definition}
It can be proven that the Shapley value is the only function that satisfies the \emph{Pareto optimality}, i.e., $\sum_{X_i\in\mathcal{F}}\phi_v(X_i)=v(\mathcal{F})$, the dummy, the symmetry and additive properties~\cite{shapley}. Moreover, Shapley values represent a fair assignment of resources to players based on their contributions to the game. We use the scores $\phi(X_i)$ to rank the features in the dataset $\mathcal{F}$. However, Shapley values do not consider redundancies, and linearly dependent features obtain equal Shapley values. 

\subsection{Importance scores of low correlated features}
We use a dataset with three sets of correlated features (color-coded in Figure~\ref{fig:ranking_synth}(a)), and we aim to select features from subsets with different colors; however, as we have already underlined, correlated features are characterized by similar Shapley values. In particular, the three highest Shapley values are obtained by correlated features in the blue-colored set. Before addressing the problem of redundancy-awareness inclusion in Shapley values, we show that the Shapley values rank features that are not correlated with the rest of the dataset in low positions.
\begin{algorithm}[!t]
	\caption{SVFS}
	\begin{algorithmic}[1]
	 \label{alg:shap}
		\Procedure{svfs}{$\mathcal{F}$, $\epsilon$}\label{alg:2ranking}
		\State $\mathcal{S} = \emptyset$
		\While{$\mathcal{F}\neq\emptyset $}
		\While {$X\in\mathcal{F}$}
		\If {$H(X)+H(\mathcal{S})-H(\mathcal{S},X)>\epsilon$}
		\State$\mathcal{F}=\mathcal{F}\setminus X$
		\Else
		\State $\mathcal{F}=\mathcal{F}$
		\EndIf
		\State $\mathcal{S} = \mathcal{S} \cup\arg\max_{X\in\mathcal{F}} \{\phi(X)\}$
		\State $\mathcal{F}=\mathcal{F} \setminus \mathcal{S}$
		\EndWhile
		\EndWhile
		\EndProcedure
		\Return{$\mathcal{S}$}
	\end{algorithmic}
\end{algorithm}%

\begin{algorithm}[!t]
	\caption{SVFR}
	\begin{algorithmic}[1]
		\Procedure{svfr}{$\mathcal{F}$}
		\State $\mathcal{S}= \arg\max_{X\in\mathcal{F}} \{\phi(X)\}$
		\State ordered = $[\ ]$
		\State ordered$[0] = \arg\max_{X\in\mathcal{F}} \{\phi(X)\}$, \quad j = 1
		\State $\mathcal{F} = \mathcal{F}\setminus\mathcal{S}$
		\While{$\mathcal{F}\neq\emptyset\ \&\&\ j < N$}
		\For {$X\in\mathcal{F}$}
		\State $\text{rk}(X) = \phi(X)-H(X)-H(\mathcal{S})+H(\mathcal{S},X)$
		\EndFor
		\State ordered$[j]=$ 
		$\arg\max_{X\in\mathcal{F}}\{\text{rk}(X)\}$
		\State $\mathcal{S} = \mathcal{S}\cup\arg\max_{X\in\mathcal{F}}\{\text{rk}(X)\}$
		\State $\mathcal{F}=\mathcal{F}\ \setminus \mathcal{S}$
		\State$j ++$
		\EndWhile
		\EndProcedure
		\vspace*{0.1cm}
		\Return{ordered}
	\end{algorithmic}
\end{algorithm}%
\begin{theorem}
	\label{thm:null_shapley}
	Given a subset of features $\mathcal{B}\subset\mathcal{F}$ that satisfies the following properties
	\begin{enumerate}
 	\item for all $ X_j\notin \mathcal{B}$ and for all $ \mathcal{A}\subseteq\mathcal{F}\setminus\{X_j\}$, $H(\mathcal{A})+H(X_j)=H(\mathcal{A}\cup X_j)$
 	\item for all $ X_i\in \mathcal{B}$ and for all $ \mathcal{A}\subseteq \mathcal{F}\setminus\{X_i\}$, 
 	$H(\mathcal{A})+H(X_i) \geq H(X_i\cup\mathcal{A})$ 
	\end{enumerate}
 then $\phi(X_i)\geq\phi(X_j)$ for all $X_i\in \mathcal{B}$ and $ X_j\notin \mathcal{B}$.
\end{theorem}
\begin{proof}From (1) we know that, since the marginal contribution of $X_j\notin \mathcal{B}$ to any $\mathcal{A}\subseteq\mathcal{F}\setminus\{X_j\}$ is equal to zero, $\phi(X_j)=\sum_{\mathcal{A}\subseteq{\mathcal{F}\setminus\{X_j\}}}\frac{1}{N\binom{N-1}{|\mathcal{A}|}}\cdot 0=0.$ \par 
For any $X_i\in\mathcal{F}$ and $\mathcal{A}\subseteq\mathcal{F}$, we know that $H(\mathcal{A}\cup X_i)\leq H(\mathcal{A})+H(X_i)$ from Shannon entropy's properties~\cite{shannon}. Hence, all marginal contributions are non-negative. Hence, $\phi(X_i)\geq0= \phi(X_j)$ for all $X_i\in \mathcal{B}$ and $X_j\notin \mathcal{B}$. \par 
This concludes the proof.

\end{proof}
\vspace{0.1cm}

Thus with total correlation as value function, Shapley values are non-negative and equal zero if and only if the feature is non-correlated with any subset of features. Moreover, features highly correlated with other subsets of features get high Shapley values.\par
\section{Redundancy removal}
We address the challenge of adding redundancy awareness to Shapley values. For this purpose, we develop a pruning criteria based on the total correlation and greedily rank features to get a redundancy-free ranking of features while still looking for features with high Shapley values. Feature selection based on this ranking selects the variables ranked first by Shapley values which show little dependencies. \par
We propose two algorithms. The Shapley Value Feature Selection (SVFS) needs a parameter $\epsilon$ representing the correlation among features that we are willing to accept; hence, SVFS requires some expert knowledge on the dataset to specify the parameter $\epsilon$ in an opportune interval. The Shapley Value Feature Ranking (SVFR) works automatically with an included notion of redundancy. We show that the two algorithms lead to consistent results in Section~\ref{sec:consistency}. At each step, both algorithms select the highest-ranked feature among the ones left. \par
We use a total correlation-based punishment; In particular, $H(\mathcal{A})+H(X)-H(\mathcal{A}\cup X) \geq 0$ represents the strength of the correlation among $X$ and $\mathcal{A}$ and it is equal to zero if and only if $X$ and $\mathcal{A}$ are independent. \par
SVFS's inputs are the set of unordered features $\mathcal{F}$ and the parameter $\epsilon >0$; $\epsilon$ plays the role of a stopping criterion and represents the maximum correlation that we are willing to accept within the set of selected features. Whenever $\epsilon$ is high, we end up with the ordering given by Shapley values alone; instead, for $\epsilon\approx0$ the criterion can lead to the selection of the only features which are uncorrelated with the first one. The optimal range of $\epsilon$ highly depends on the dataset. We show that SVFS is robust w.r.t. the choice of $\epsilon$. At each iteration, SVFS excludes from the ranking the features $X$s that are correlated with the already ranked features $\mathcal{S}\subseteq \mathcal{F}$ more than $\epsilon$, i.e., $H(X)+H(\mathcal{S})-H(\mathcal{S},X)>\epsilon$, computes the Shapley values of all remaining features $X$ and adds to $\mathcal{S}$ the feature whose Shapley value is the highest. When there are no features left, it stops and returns $\mathcal{S}$. 
\par
SVFR takes as an input $\mathcal{F}$ and outputs a feature ranking without the need of any additional parameter. The ranking is aware of correlations as each of the Shapley values $\phi(X_i)$ is penalized using the correlation measure $H(X_i)+H(\mathcal{S})-H(X_i \cup \mathcal{S})$ where $\mathcal{S}$ is the set of already ranked features, and $X_i$ is a new feature to be ranked. This algorithm provides a complete ranking of features and can be prematurely stopped including an upper bound of features we are willing to rank. The absence of the additional parameter $\epsilon$ is the main advantage of SVFR over SVFS. 

\section{Scalable algorithms}\label{sec:approx}
The size of $\mathcal{P}(\mathcal{F})$ being exponential in $N$, computing Shapley values involves $2^N$ evaluations of the value function. We use approximated Shapley values to obtain scalable versions of SVFR and SVFS. We implement three versions of the algorithms that differ only in the computations of Shapley values used:
\begin{itemize}
    \item \emph{full algorithm}: it uses the full computation of the Shapley values
    \item \emph{bounded algorithm}: consider only subsets up to size $k$ fixed to compute the Shapley values
    \item \emph{sampled algorithm}: it uses the approximation proposed by Castro et al.~\cite{approx2} based on $n$ random sampled subsets of features.
\end{itemize} 
The time complexity for the sampled algorithm is $\mathcal{O}(D\cdot n)$, for the bounded algorithm is $\mathcal{O}(D\cdot N^{k})$ while for the full algorithm is $\mathcal{O}(D\cdot2^N)$ where $N$ is the number of features and $D$ the number of samples in the dataset.
\section{Experiments}\label{sec:experiments}
We show that our feature ranking method outperforms competing representative feature selection methods in terms of redundancy reduction. Metrics such as NMI, ACC, and {redundancy rate} are often used in the previous literature to evaluate unsupervised feature selection methods. NMI and ACC focus on the cluster structure in the data; therefore, as clustering is not the goal of our approach, we compare it with the competing methods using the redundancy rate. The redundancy rate of $\mathcal{S}\subseteq\mathcal{F}$ is defined in terms of pairwise Pearson correlations, i.e.,
\begin{equation}\label{eq:redundancy}
 \text{Red}(\mathcal{S})=\frac{1}{2m(m-1)}\sum_{X,Y\in \mathcal{S}, X\neq Y} \rho_{X,Y}
\end{equation}
where $\rho_{X,Y}\in[0,1]$ is the Pearson correlation of features $X$ and $Y$. It represents the averaged correlation among the pairs of features in $\mathcal{S}$ and varies in the interval $[0,1]$: a $\text{Red}(\mathcal{S})$ close to $1$ shows that many selected features in $\mathcal{S}$ are strongly correlated while a value close to zero indicates that $\mathcal{S}$ contains little redundancy. In the experiments, we use the \emph{redundancy rate} as evaluation criteria re-scaling it to the interval $[0,100]$ via the maximum pair-wise correlation to facilitate the comparison among different datasets.

\subsection{Datasets and competing methods}
We show a comparison against \emph{SPEC}, \emph{MCFS}, \emph{UDFS}, \emph{NDFS}, \emph{PFA}, \emph{LS} and FSFC~\cite{SPEC,MCFS,UDFS,NDFS,PFA,LS,FSFC}.\par 
We use various synthetic and publicly available datasets: the \href{https://archive.ics.uci.edu/ml/datasets/breast+cancer}{\underline{\emph{Breast Cancer dataset}}}, the \href{https://www.kaggle.com/tunguz/big-five-personality-test}{\underline{\smash{\emph{Big Five Personalities Test dataset}}}}\footnote{The first $50$ features in the Big Five dataset are the categorical answers to the personality test's questions and are divided into $5$ personalities' traits ($10$ questions for each personality trait). To apply the full algorithm, we select questions from different personalities and restrict to $10000$ instances.} and the \href{https://www.kaggle.com/datasets/stefanoleone992/fifa-21-complete-player-dataset?select=players_20.csv}{\underline{\emph{FIFA dataset}}}\footnote{We restrict to the $5000$ highest-rated players by the overall attribute.}. The datasets that we use throughout the paper are all categorical or discrete. We consider subsets of the full dataset in order to apply the full versions of the algorithms and investigate the performance of the approximations of SVFR and SVFS at the end of the section.

\subsection{Redundancy awareness}
We compare the feature selection results of our algorithm against the competitors by evaluating the redundancy rate in Table~\ref{tab:redundancy}. For the FIFA dataset, we select $15$ features from the entire data which characterize the \emph{agility}, \emph{attacking} and \emph{defending} skills of the football players; we keep the whole datasets for Breast Cancer and synthetic data; in the case of the Big Five Personality Traits dataset, we select respectively $5$ questions from three different personality traits for the balanced dataset and $9$ features from one trait and $3$ from other two personality traits in the case of the unbalanced dataset. In order to avoid bias towards the random selection of personality traits and features in the Big Five data, we average the redundancy rate over $30$ trials on randomly selected personalities and variables both in the case of the balanced and unbalanced setup.\par 
In each column, bold characters highlight the lowest redundancy rate. We use SVFR for ranking the features and select the three highest-ranked features. We consequently specified the parameters of the competing methods in order to get a selection of features as close to three features as possible. For FCFS we set $k=4$ for BC dataset, $k = 8$ for FIFA dataset, $k = 8$ for the synthetic data and for \emph{Big Five} dataset we use different $k$ at each re-run such that the number of selected variables varies between $2$ and $5$ and then we average the redundancy rates; for NDFS, MCFS, UDFS and LS we used $k=5$ ($k$ being the number of clusters in the data); for the other competitors, we specify the number of features to be selected. Table~\ref{tab:redundancy} illustrates that SVFR outperforms the competing methods in nearly all the cases. In particular, while SVFR achieves low redundancy rates in all datasets, the competing algorithms show big differences in performance in the various datasets. On the Breast Cancer data and the synthetic dataset respectively, PFA and NDFS slightly outperform SVFR. However, they do not keep an average low redundancy rate on the other datasets.  For reproducibility, we make the code publicly available \footnote{\url{https://github.com/chiarabales/unsupervised_sv}}.

\begin{table*}[!t]
\centering
\begin{tabularx}{\textwidth} { >{\raggedleft\arraybackslash}X | >{\raggedleft\arraybackslash}X >{\raggedleft\arraybackslash}X >{\raggedleft\arraybackslash}X >{\raggedleft\arraybackslash}X >{\raggedleft\arraybackslash}X }

& Breast Cancer & 
B5\_balanced & 
B5\_unbalanced & 
FIFA & 
Synthetic 
\vspace*{0.1mm} \\ \hline

NDFS & 36.30 & 
22.11 & 
20.75 & 
18.97 & 
\textbf{1.49}
\\

MCFS & 20.26 & 
23.59 & 
18.79 &
20.63 & 
3.74
\\

UDFS & 33.59
& 28.13 
& 35.18 
& 57.73 
& 4.06
\\

SPEC 
& 13.89 
& 39.09 
& 21.46 
& 42.14 
& 29.4 
\\

LS 
& 7.05 
& 28.83 
& 58.25 
& 48.28 
& 100.00 
\\

PFA 
& \textbf{5.10} 
& 23.22 
& 34.28 
& 57.42 
& 35.84 
\\

FSFC 
&8.74 & 
22.64 
& 20.99 
& 36.45 
& 2.12 
\\

\rowcolor{colorlight} {SVFR} 
& {6.68} 
& {\textbf{15.65}} 
& {\textbf{18.02}} 
& {\textbf{14.79}} 
& {{1.51}} 
\end{tabularx}	
 \vspace*{3mm}
 \caption{\label{tab:redundancy}Redundancy rate of the sets of three selected features using the competing algorithms and SVFR (highlighted in green color in the table) on different datasets. The lowest rates are represented in bold characters.}
	\vspace{-5mm}
\end{table*}

\begin{figure}[!t]
\subcaptionbox{\label{fig:ranking_BC}}%
 [0.6\linewidth]{\includegraphics[height=2.5cm]{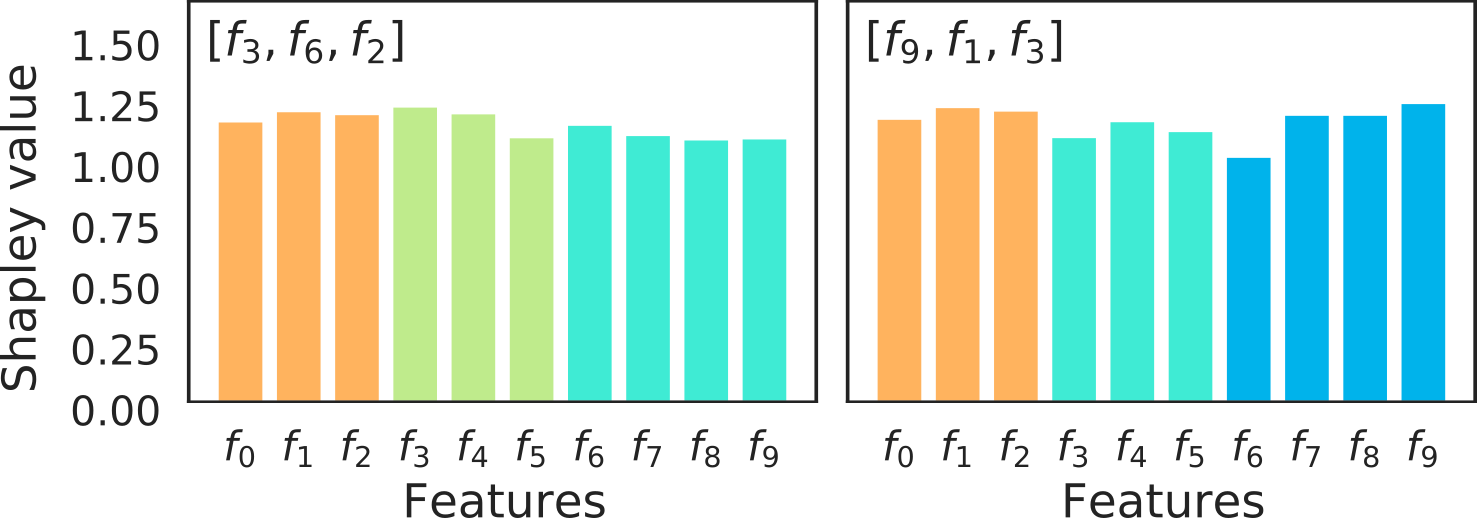}}
\subcaptionbox{\label{fig:shap_big5}}
 [0.37\linewidth]{\includegraphics[height=3.1cm]{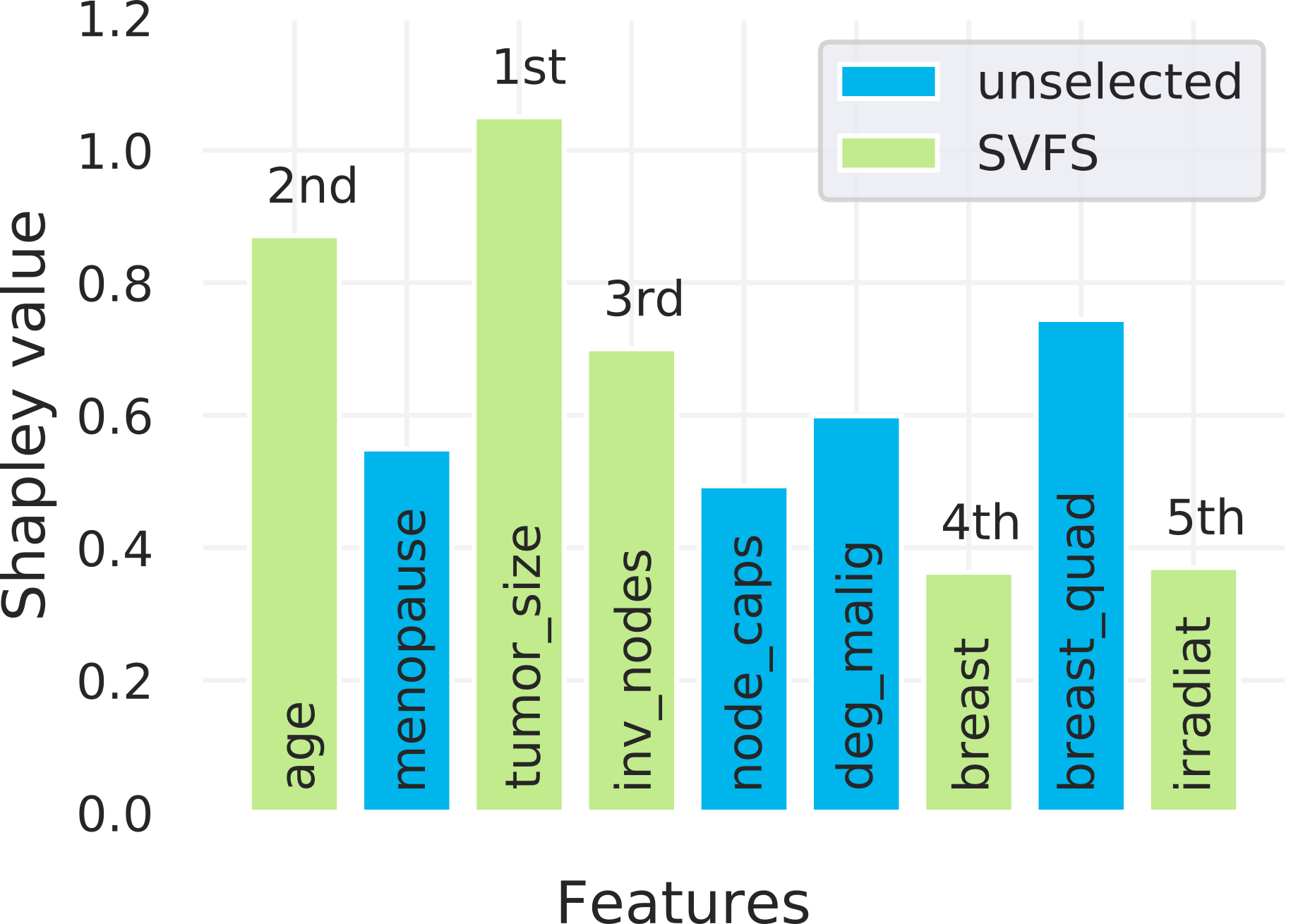}}
 \vspace*{-3mm}
 \caption{\label{fig:shap_big5_and_ranking_BC}(a) Barplot of Shapley values and respective feature selections by SVFS ($\epsilon=0.3$) in Big Five dataset restricted to $10$ features. Different personalities traits are color-coded in each plot. (b) Barplot of Shapley values for Breast cancer data; in green, the ordering of features' selection by SVFS when $\epsilon=0.5$.}
 \vspace{-2mm}
\end{figure}
\subsection{Relevance of unsupervised feature selection and effectiveness}
In Figure~\ref{fig:shap_big5_and_ranking_BC}(a), each plot corresponds to a different subset of features of the Big Five dataset, i.e., $10$ features selected from three different personality traits. Running SVFS with $\epsilon =0.3$ we detect correlated features and avoids selecting them together as shown in the plots. Using the scaled versions of our algorithms from Section~\ref{sec:approx} we can extend the approach towards the complete Big Five dataset. \par 
Figure~\ref{fig:ranking_synth}(a) represents the Shapley values of features in a $12$ dimensional synthetic dataset where subsets of correlated features are color-coded. We measure the ability of the algorithm in selecting features from different subsets of correlated features; SVFS selects one feature from each subset of correlated features. In particular, when $\epsilon=1$, SVFS achieves this goal by selecting $\{f_8, f_{7},f_{3}\}$ while the ranking given by the Shapley values alone is $\{f_8, f_{10},f_{11}\}$ which belong to the same subset of correlated features. This nicely underlines the inability of Shapley values to detect correlations and the necessity of integrating correlation-awareness to perform a feature selection.

\begin{table*}[!t]
	\begin{tabularx}{\textwidth} { >{\raggedright\arraybackslash}X | >{\raggedright\arraybackslash}X >{\raggedright\arraybackslash}X >{\raggedright\arraybackslash}X}
		
		& Big Five & Synthetic Data & Breast Cancer\\\hline 
        $\epsilon = 0.2$ & $[\textcolor{colorA}{11}, \textcolor{colorE}{0}, \textcolor{yellow}{5}]$ & $[8, \textcolor{colorH}{7}, \textcolor{colorE}{0}]$ & $[\textcolor{colorD}{2}, \textcolor{colorE}{0}, 8]$ \\
        $\epsilon = 0.3$ & $[\textcolor{colorA}{11}, \textcolor{colorE}{0}, \textcolor{colorG}{10}]$ & $[8, \textcolor{colorH}{7}, \textcolor{yellow}{2}]$ & $[\textcolor{colorD}{2}, \textcolor{colorE}{0}, \textcolor{colorF}{4}, \textcolor{colorG}{6}]$ \\
        $\epsilon = 0.4$ & $[\textcolor{colorA}{11}, \textcolor{colorE}{0}, \textcolor{colorH}{14}]$ & $[8, \textcolor{colorH}{7}, \textcolor{colorA}{3}]$ & $[\textcolor{colorD}{2}, \textcolor{colorE}{0}, \textcolor{colorF}{4}, 8, \textcolor{colorG}{6}]$ \\
        $\epsilon = 0.5$ & $[\textcolor{colorA}{11}, \textcolor{colorE}{0}, \textcolor{colorH}{14}, \textcolor{colorB}{9}]$ & $[8, \textcolor{colorH}{7}, \textcolor{colorA}{3}]$ & $[\textcolor{colorD}{2}, \textcolor{colorE}{0}, \textcolor{colorA}{3}, 8, \textcolor{colorG}{6}]$ \\
        $\epsilon = 0.6$ & $[\textcolor{colorA}{11}, \textcolor{colorE}{0}, \textcolor{colorH}{14}, \textcolor{yellow}{5}]$ & $[8, \textcolor{colorH}{7}, \textcolor{colorA}{3}]$ & $[\textcolor{colorD}{2}, \textcolor{colorE}{0}, \textcolor{colorA}{3}, 8, \textcolor{colorG}{6}]$ \\
        $\epsilon = 0.7$ & $[\textcolor{colorA}{11}, \textcolor{colorE}{0}, \textcolor{colorH}{14}, 13]$ &$[8, \textcolor{colorH}{7}, \textcolor{colorA}{3}]$ & $[\textcolor{colorD}{2}, \textcolor{colorE}{0}, \textcolor{colorA}{3}, \textcolor{colorF}{4}, 8, \textcolor{colorG}{6}]$ \\
        $\epsilon = 0.8$ & $[\textcolor{colorA}{11}, \textcolor{colorE}{0}, \textcolor{colorH}{14}, 13]$ & $[8, \textcolor{colorH}{7}, \textcolor{colorA}{3}, \textcolor{colorE}{0}]$ & $[\textcolor{colorD}{2}, \textcolor{colorE}{0}, \textcolor{colorA}{3}, \textcolor{colorF}{4}, \textcolor{yellow}{5}, 8, \textcolor{colorG}{6}]$ \\
        SVFR & $[\textcolor{colorA}{11}, \textcolor{colorE}{0}, \textcolor{yellow}{5}, \textcolor{colorG}{10}, 12, 8, \textcolor{colorG}{6}, \textcolor{colorD}{2}]$ & $[8, \textcolor{colorH}{7}, \textcolor{colorA}{3}, \textcolor{colorE}{0}, \textcolor{colorG}{6}, {5}, \textcolor{yellow}{2}, \textcolor{colorB}{10}]$ & $[\textcolor{colorD}{2}, \textcolor{colorE}{0}, \textcolor{colorF}{4}, \textcolor{colorG}{6}, 8, \textcolor{yellow}{5}, \textcolor{colorB}{1}, \textcolor{colorA}{3}]$
	\end{tabularx}
	\vspace*{2mm}
	\caption{\label{tab:new_algorithm}Orderings of selection given by SVFS for various $\epsilon$ and first $8$ ranked features by SVFR. Features are color-coded in order to simplify the visualization. 
	}
	\vspace{-7mm}
\end{table*}
Our unsupervised feature selection allow to construct more efficient psychological tests avoiding redundancies and reducing the number of questions that need to be answered without losing too much information.
\begin{table}[!t]
    \centering
    \setlength\extrarowheight{1pt}
    \begin{tabularx}{\linewidth} { >{\raggedright\arraybackslash}X | >{\raggedright\arraybackslash}X >{\raggedright\arraybackslash}X >{\raggedright\arraybackslash}X >{\raggedright\arraybackslash}X}
    & & $k = 1$ & $k = 3$ & $k = 5$ \vspace*{0.1mm} \\ 
    \hline 
    \multirow{3}{*}{\begin{sideways} BIG5 \end{sideways}}& random & 0.04 & 0.19 & 0.33 \\
    & sampled & 0.04 & 0.37 & 0.49\\
    & bounded & \textbf{0.08} & \textbf{0.56} & \textbf{0.55}\vspace*{0.1mm} \\
    \hline\multirow{3}{*}{\begin{sideways} FIFA \end{sideways}} & random & 0.06 & 0.24 & 0.35\\
    & sampled & 0.00 & 0.33 & 0.40 \\
    & bounded & \textbf{1.00} & \textbf{0.67} & \textbf{0.80}
    \end{tabularx}
    \vspace*{1mm}
    \caption{$recall@k$ for $k\in \{ 1, 3, 5\}$ comparing a random ranking and the rankings given by SVFR using the sampled and bounded algorithms to the full SVFR ranking. We show results for FIFA and Big Five datasets restricted to $15$ features randomly chosen. Bold text highlights the best approximation.}
    \label{tab:firstk}
    \vspace{-0.4cm}
\end{table}

\subsection{Interpretation of feature ranking}
We apply SVFS when $\epsilon=0.5$ to the Breast Cancer dataset. In Figure~\ref{fig:shap_big5_and_ranking_BC}(b), the resulting Shapley values and the ordering of selected features are displayed. The selection resulting from SVFS shows a low redundancy rate while the selected features (e.g., the size of the tumour, age, and the number of involved lymph nodes) are clearly in line with domain knowledge on risk factors for disease progression (label). Furthermore, the comparison with the ranking without redundancy awareness nicely highlights the importance of our approach to avoid redundancies when possible. 

\subsection{Comparison among the proposed algorithms}\label{sec:consistency}
In Figure~\ref{fig:change_eps_graph}, we plot a comparison among SVFS and SVFR w.r.t. the redundancy rate on three datasets with different values of $\epsilon$. As benchmarks, we use for SVFR the selection of $3$, $4$ and $5$ features respectively while for SVFS, $\epsilon$ varies in the interval $[0,1.4]$ with steps of size $0.1$. \par
Using the number of features as a stopping criterion in SVFR would produce consistent results to SVFS: as an example, using the breast cancer data the ranking given by SVFR, i.e., $[{2}, {0}, {4}, {6}, 8, {5}, {1}, {3}]$, is consistent with the selection given by SVFS respectively using $\epsilon=0.2$ and $\epsilon = 0.6$, i.e., $[{2}, {0}, 8]$ and $[{2}, {0}, {3}, 8, {6}]$. Table~\ref{tab:new_algorithm} shows a full comparison among the SVFR and SVFS on three different representative datasets. We recommend applying SVFS when no previous knowledge of the data is available and it is hard to establish an optimal range for $\epsilon$. Vice versa, one could apply SVFR when the expertise in the dataset domain allows determining a reasonable number of features as stopping criterion or the observation of the ranking given can provide insights to the non-expert on which features to keep and which can be discarded for further analysis.


\begin{figure*}[!t]
	\centering
	\includegraphics[width=\linewidth]{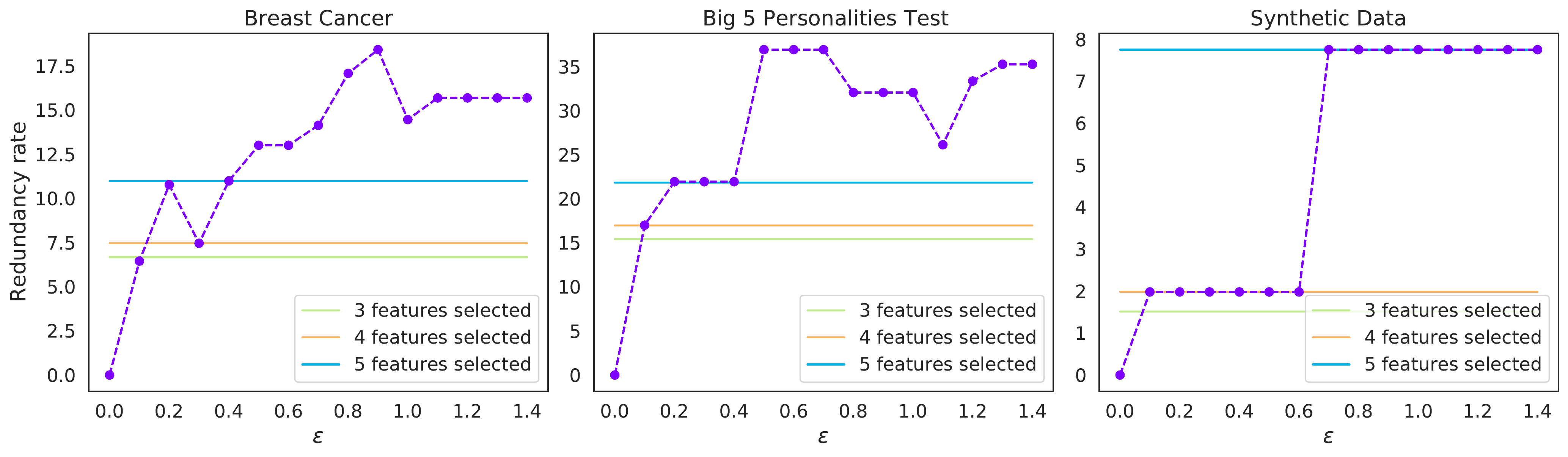}
	\caption{Redundancy rates of the selected features' sets as a function of $\epsilon$ for SVFS (bullets points connected by the dashed line) and for $3,4,$ and $5$ selected features when using SVFR.}
	\label{fig:change_eps_graph}
	\vspace{-5mm}
\end{figure*}

\subsection{Run-time analysis}\label{sec:runtimeanalysis}
As a consequence of the full computation of Shapley values, the run-time of SVFR and SVFS increases exponentially with the number of features as shown by Figure~\ref{fig:runtime}. Using the approximated algorithms, this growth turns out to be slower. In particular, when using the sampled algorithm, the run-time increases only linearly with the number of features while the growth of the bounded algorithm's run-time is polynomial in the number of features. In the additional material, we show the log-log plot of the run-time for increased number of samples in the dataset. For each algorithm, we use random subsets of the Big Five dataset and average over $10$ trails. \par
We further compare the rankings of the approximated and full algorithms using the $recall@k$ metric interpreting rankings of the full version of SVFR as ground truth. We use the Big Five dataset, randomly selecting $5$ questions from $3$ different personalities and average the scores over $100$ trails (see Table~\ref{tab:firstk}). Overall, the results for the approximated algorithms clearly outperform random ordering - but still deviate often from the full versions. It is worth to note that the bounded algorithm using subsets up to size $5$ performs better than the sampled version. 

\begin{figure}[!t]
 \centering\includegraphics[width=0.90\textwidth]{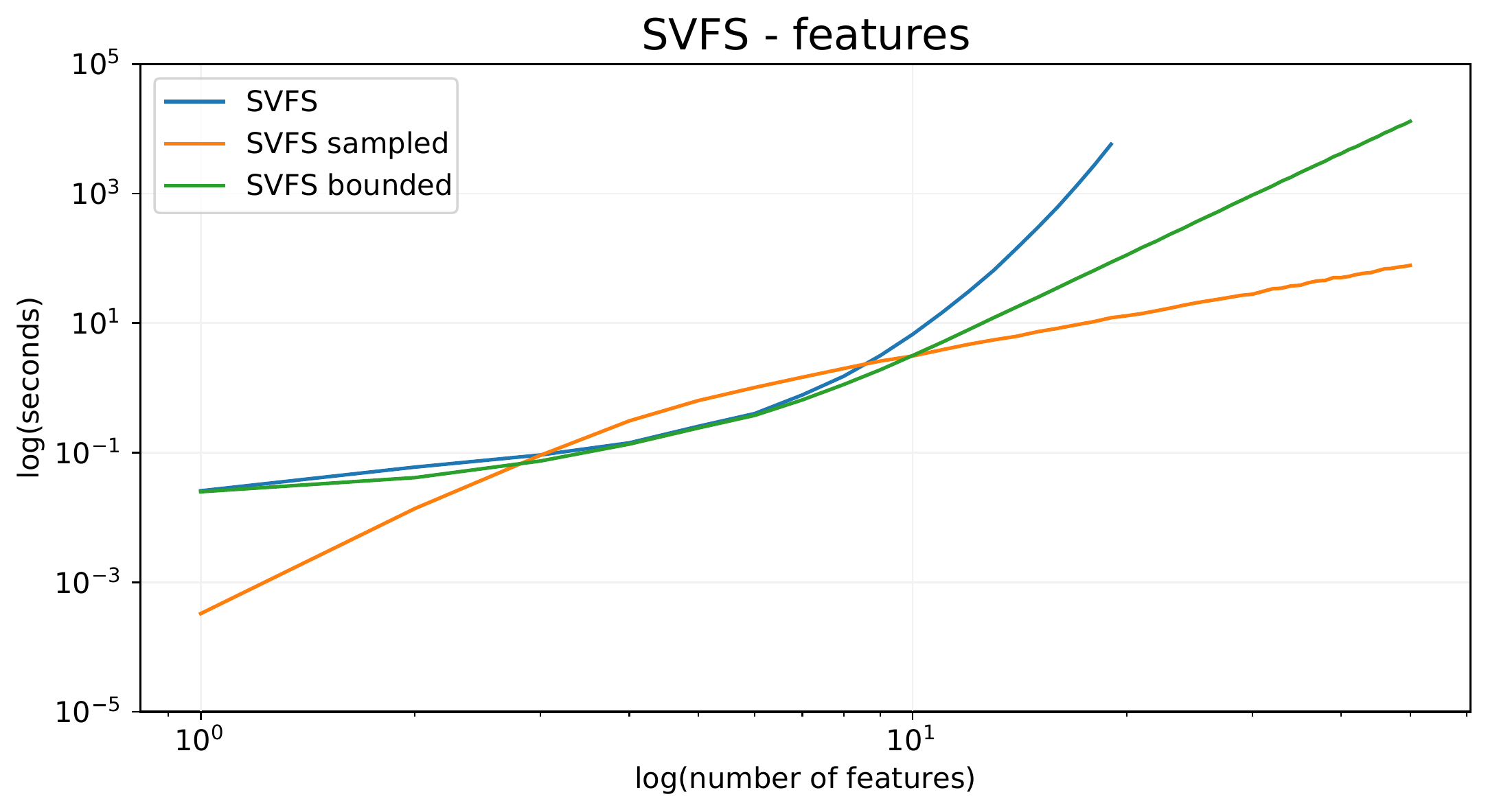}
 \vspace*{-2.5mm}
 \caption{\label{fig:runtime} Log-log plots of the run-time as a function of the number of features for the approximated and full SVFS ($\epsilon=0.5$, $D = 1000$). The full SVFS is stopped with $20$ features.}
 \vspace{-6mm}
\end{figure}


\section{Conclusions}
In the paper, we develop a new method to assess feature importance scores in unsupervised learning, bridging the gap between unsupervised feature selection and cooperative game theory. We integrate Shapley values with redundancy awareness making use of an entropy-based function to get feature importance scores. \par 
We present two algorithms: SVFS implements feature selection using a redundancy aware criterion while SVFR assigns a ranking to each feature while being aware of correlations with previously ranked features. We show how the results of the two algorithms are consistent and state-of-the-art regarding their application. Our feature selection methods outperform previously proposed algorithms w.r.t. the redundancy rate. We additionally introduce approximated versions of the algorithms that are scalable to higher dimensions.

\newpage

\bibliography{arxiv_biblio} 
\bibliographystyle{siam}
\section*{Additional material}
\begin{figure}
 \centering
 \includegraphics[width=0.9\textwidth]{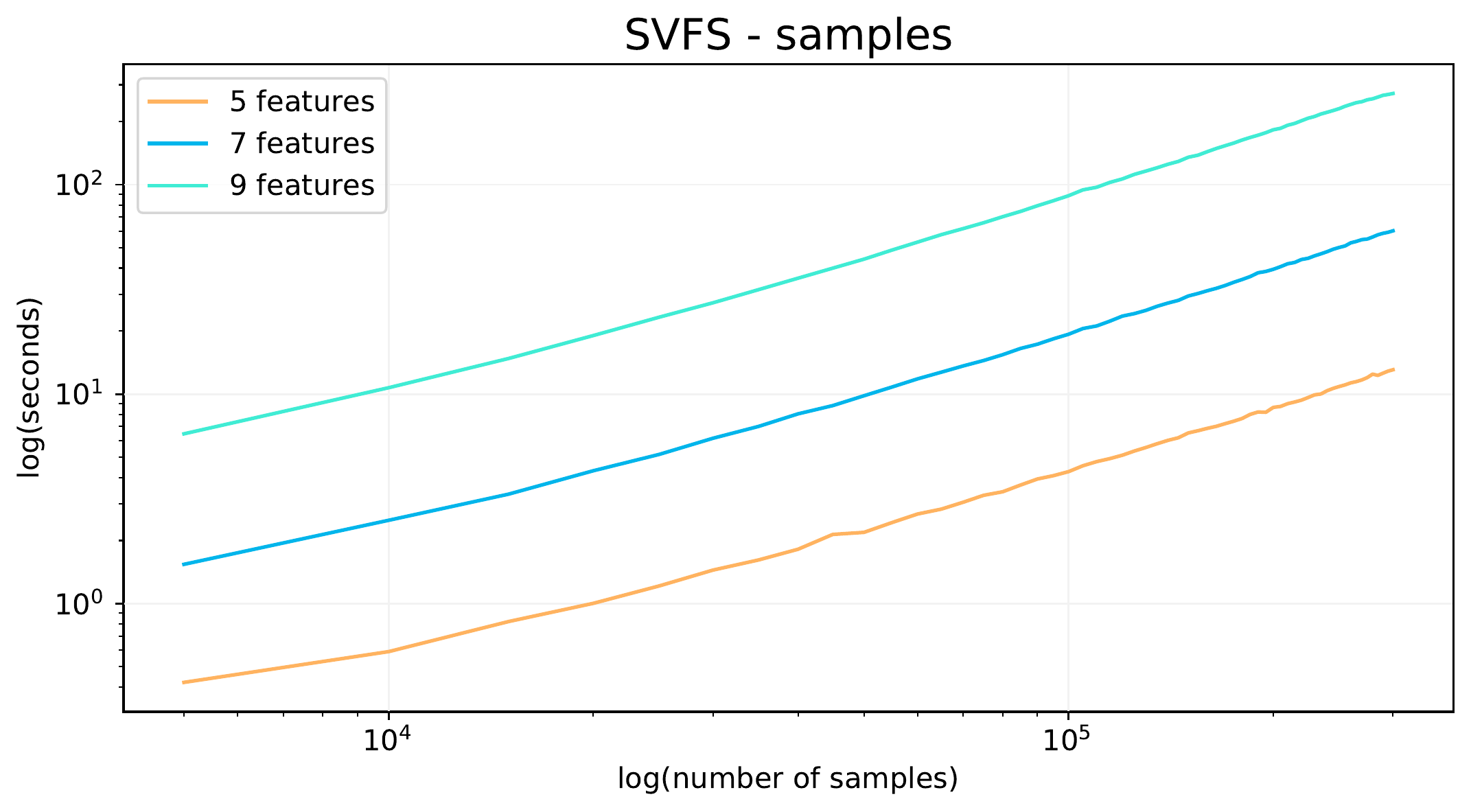}
 \caption{\label{fig:runtimesamples}Log-log plot of the run-time for the full SVFS with $\epsilon=0.5$ as a function of the number of the samples $D$ and fixed number of features.}
 \vspace{-1cm}
\end{figure}

\begin{table*}[!h]
 \centering
 \setlength\extrarowheight{1pt}
 \begin{tabularx}{\linewidth} { X|  X X  }
 & features & samples  \\ \hline
Breast Cancer dataset & 9 & 286 \\
Big Five dataset & 50 & 1013558 \\
FIFA20 dataset &  46 & 15257 \\
synthetic dataset & 12 & 10000
 \end{tabularx}
 \vspace*{1mm}
 \caption{Summary of the datasets' structures.}
 \vspace{-0.4cm}
\end{table*}

\end{document}